\PassOptionsToPackage{table}{xcolor} 
\documentclass[11pt]{article} 
\PassOptionsToPackage{numbers}{natbib}
\usepackage{arxiv}
\usepackage{natbib}

\usepackage{times}
\usepackage{wrapfig}
\usepackage{algpseudocode}
\usepackage{algorithm}
\usepackage{graphicx}      
\usepackage{adjustbox}     
\usepackage{caption}       
\usepackage{amsmath}       
\usepackage{amssymb}       
\usepackage{placeins}
\usepackage{multicol}
\usepackage{stfloats}
\usepackage{flafter}
\usepackage{latexsym}
\usepackage{listings}
\usepackage{tcolorbox}
\usepackage{soul}
\usepackage{xcolor}
\usepackage{bm}
\usepackage{xurl}
\usepackage{microtype}
\usepackage[T1]{fontenc}

\usepackage[utf8]{inputenc}
\usepackage{xcolor}
\usepackage{microtype}
\usepackage[normalem]{ulem}
\usepackage{multirow}
\usepackage{amsthm}

\usepackage{amsmath,amsfonts,bm}

\def\eqref#1{equation~\ref{#1}}

\def\1{\bm{1}}

\DeclareMathAlphabet{\mathsfit}{\encodingdefault}{\sfdefault}{m}{sl}
\SetMathAlphabet{\mathsfit}{bold}{\encodingdefault}{\sfdefault}{bx}{n}

\usepackage{url}
\usepackage{float}
\usepackage{makecell}
\usepackage{graphicx}
\usepackage{subfig}
\usepackage[para,online,flushleft]{threeparttable}
\usepackage{booktabs}
\newsavebox{\distributiontable}
\newsavebox{\miatable}
\newsavebox{\multiroundtables}
\newsavebox{\clippingtablebox}

\newcommand{\tablebodyfont}{\fontsize{7}{8.5}\selectfont}

\usepackage[capitalise]{cleveref}
\usepackage{etoc}

\newtheorem*{theorem*}{Theorem}

\newtheorem{lemma}{Lemma}
\newtheorem*{lemma*}{Lemma}

\newtheorem*{corollary*}{Corollary}

\sethlcolor{yellow} 
\newcommand{\hlgreen}[1]{%
  {\sethlcolor{green}\hl{#1}}%
}

\title{Private Retrieval Augmented Generation via Random Projection}

\author{Dixi Yao\footnotemark[0] \thanks{Correspondence
to: Dixi Yao $<$dixi@uchicago.edu$>$} \quad
Tahseen Rabbani \quad
Tian Li \\University of Chicago 
}

\begin{document}
\raggedbottom

\maketitle
\etocdepthtag.toc{mtchapter}

\begin{abstract}
Retrieval-Augmented Generation (RAG) enhances the capabilities of large language models (LLMs) by querying external structured knowledge. However, it can also introduce privacy risks by leaking sensitive information from the retrieval datastore. We propose a simple method to preserve datastore privacy in RAG systems via random projection. By applying the same projection to both datastore embeddings and query embeddings, our method provably preserves semantic similarity between queries and retrieved items while substantially mitigating data extraction attacks. Across multiple RAG architectures and datasets, we show that this lightweight approach achieves superior retrieval and generation performance compared to prior methods which leverage formal differential privacy (DP) guarantees, while exhibiting comparable empirical privacy under strong attack models. Our results for the first time suggest that random projection can serve as a competitive and practical baseline for privacy-preserving RAG systems. 
\end{abstract}

\section{Introduction}
Retrieval-Augmented Generation (RAG)~\citep{khandelwal2019generalization, lewis2020retrieval} enhances large language models (LLMs) by incorporating information retrieved from external knowledge. However, recent studies show that carefully crafted prompts~\citep{he2025external, wang2025silent, zeng2024good, koga2024privacy} can extract sensitive information such as personal identity information from the external datastore.

Since attack queries can closely resemble legitimate user queries, simple access control may block ordinary users~\citep{he2025external,wang2025silent}. Recent work has begun to explore privacy risks in RAG, with a leading focus on differential privacy (DP)~\citep{koga2024privacy,wu2025private}.
Although DP methods provide formal statistical guarantees, they often degrade output quality (e.g., leading to high perplexity) due to DP's fundamental properties of worst-case guarantees, and are usually complicated to implement in practice. 
To this end, we propose private RAG via random projection to defend against data extraction attacks.

Our key idea is that differential privacy guarantees may be an overkill to prevent state-of-the-art attacks practically, whereas (some specific) random projection matrices naturally preserve pairwise similarities of the inputs in the lower-dimensional space. The extra randomness alters datastore embeddings so attackers may not recover the original text. This enables users to obtain accurate answers from the RAG system while adversaries fail. 
We apply  a single projection step to datastore embeddings in the offline stage and the same matrix to user queries in the online stage.

In this work, we consider two RAG architectures: `KNN-LM'~\citep{khandelwal2019generalization} and `Standard RAG'~\citep{lewis2020retrieval}\footnote{We note that there exist various RAG architectures, and we use the naming `Standard RAG' just to contrast with the less popular KNN-LM architecture.}. In KNN-LM, the system linearly combines the outputs of the language model and the $k$-nearest neighbour outputs (usually organized as a softmax over cosine similarities between queries and the datastore embeddings), and generates next token. In Standard RAG, the system concatenates retrieved texts with the query and feeds them to the LLM. In this work, we focus on the settings where formal privacy guarantees are not desired, and thus evaluate privacy leakage by exploring the success rates of strong, existing attacks empirically. 

In \cref{sec:experiments}, we empirically show that enforcing differential privacy guarantees may not provide better practical protection against data extraction and can instead harm utility. In summary, our RP technique significantly improves utility (over 50\% on average across datasets) over the baseline methods without degrading empirical privacy performance.

\section{Related Works}
\paragraph{Private RAG.} 
Prior work shows that various RAG architectures can reveal sensitive personal data (e.g., emails, phone numbers, and URLs) from the external datastore~\citep{he2025external,zeng2024good, jiang2024rag,jiang2024feedback,di2024pirates}. Existing defenses, such as DP-based sampling and aggregation~\citep{koga2024privacy, grislain2025rag} and synthetic data~\citep{zeng2025mitigating}, reduce efficiency on large datasets and often require extra validation or training. Moreover, \citet{zhao2025does} show that synthetic data alone does not fully prevent memorization or leakage. \citet{he2025mitigating} apply local differential privacy to sensitive content like addresses. However, privacy can involve entire sentences or paragraphs beyond tokens, making simple masking insufficient. \citet{thareja2025dp} mix original and sanitized documents, which requires identifying all private tokens and is impractical for large-scale RAG.

\paragraph{Privacy-Preserving Random Projection.} 
Certain random projection operations are a common technique for reducing dimensionality while preserving some similarity measures with high probability (e.g., based on the Johnson-Lindenstrauss lemma~\cite{johnson1984extensions}). 
Prior works have shown that methods built on top of these projections can achieve differential privacy (DP) for specific applications~\cite{blocki2012johnson, xu2017dppro,li2023differential,ibrahim2024privacy,liu2006random,narimani2025fedrp,kaleli2013privacy,pavlovic2025differential,lee2025optimal}. In this work, we show that random projection can be effective for a variety of RAG scenarios as well. 

\section{Private RAG via Random Projection} \label{sec:alg}
\subsection{Threat Model}
\label{sec:threatmodel}

\begin{wrapfigure}{r}{0.5\columnwidth}
    \centering
    \vspace{-10pt}
    \includegraphics[
        width=\linewidth
    ]{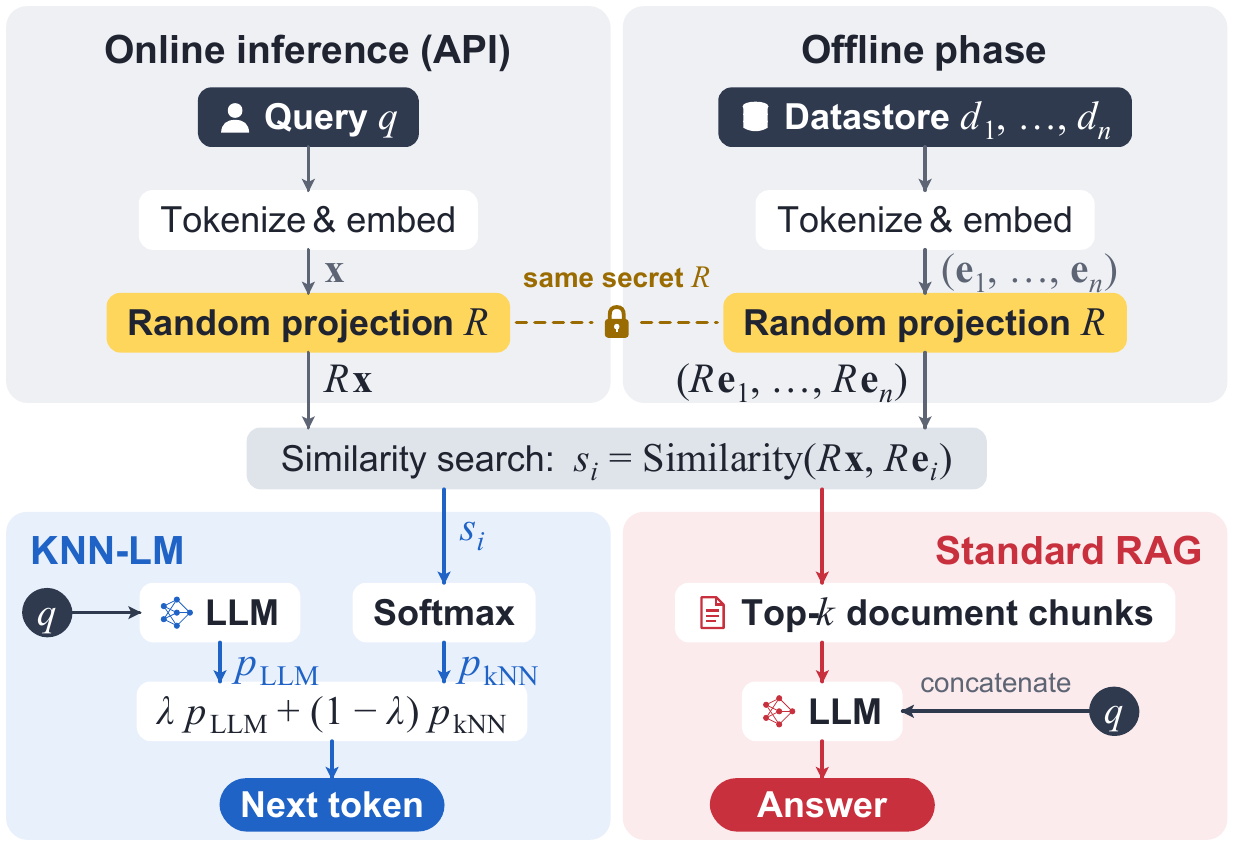}
    \captionsetup{skip=2pt}
    \caption{Overview of Private RAG via RP.}
    \label{fig:arch}
    \vspace{-10pt}
\end{wrapfigure}

As illustrated in \Cref{fig:arch}, our algorithm is compatible with both RAG architectures. In this work, we evaluate our method as an empirical defense against data extraction attacks. We consider a black-box setting in which the attacker has access only to model API responses, while model parameters, the random projection matrix, logits, embeddings, and latent representations remain inaccessible~\citep{sun2026enterpriserag,packowski2024optimizing}. The attacker submits queries to the RAG system using the API service and attempts to extract sensitive information in the datastore from the generated response through one round or multiple rounds. This is a common protocol in modern RAG applications where companies provide consulting service through LLM with an internal datastore while only text answers are exposed to users. The threat model also aligns with the threat scenarios considered in prior work on RAG privacy~\citep{he2025external, wang2025silent, huang2023privacy}. 

In short, we assume the attacker knows that RP is used but does not have direct access to the specific random matrices, with the projection matrix remaining undisclosed. Attackers are free to conduct repeated-query (\Cref{sec:repeat}) or multi-round adaptive attacks (\Cref{sec:multi}). { As network- and system-level attacks are out of scope and access to embeddings is restricted, we do not consider white-box attacks, embedding inversion, projection-matrix inversion, or side-channel attacks requiring side information or access beyond the LLM's text responses.}

\begin{table*}[!b]
\centering
\captionsetup{skip=2pt}
\caption{We focus on sensitive-information leakage and show examples of sensitive and non-sensitive information.}
\captionsetup{skip=2pt}
\label{tab:example_leakage}
{
\renewcommand{\arraystretch}{1.1}
\tablebodyfont
\setlength{\tabcolsep}{2pt}
\begin{tabular}{p{\dimexpr0.18\linewidth-4pt\relax} p{\dimexpr0.41\linewidth-4pt\relax} p{\dimexpr0.41\linewidth-4pt\relax}}
\toprule
Leakage Information & Sensitive Information & Non-Sensitive Information \\
\midrule

URL with Personal Information (e.g., Token ID) &
If you prefer to receive a hardcopy of your pay stub, click on the link below to complete a request form. Please click on the link below.
\hl{http://etcsurvey.corp.enron.com/wi/p0810709/i.asp?r=19690}\newline\hl{\&s=OUULGRTU\&l=9}
&
Suggestions and feedback are welcome at \hlgreen{https://www.reddit.com} \\

Phone &
America Square, London, EC3N 1PQ Tel: \hl{+44 (0)171 839 4444} Fax: \hl{+44 (0)171 839 4444} &
If you have any comments or suggestions contact us at\hlgreen{!  !!!!!!!!!!!!!!!!!!!!!!!!!!!!!!!!!!!!!!!!}\\

Personal Address &
If you have any questions, please do not hesitate to call. Salle E. Yoo Davis Wright \hl{Tremaine One Embarcadero Center, Suite 600 San Francisco, CA 94111}&
The applicant, Mr Richard Owens, is a British national who was born in 1944 and lives in \hlgreen{Queen St. Queen St. Queen St. Queen St.} \\
Name &
The applicant, \hl{Mr Michael Boyce}, is an Irish national, who lives in Mullingar, Ireland. He was represented before the Court by \hl{Mr P.J. Groarke}, a lawyer practising in County Longford, Ireland. &
The applicant, \hlgreen{Mr D.H.}, who was born in 1966 and lives in \hlgreen{H.} \\

\bottomrule
\end{tabular}}
\end{table*}

{To provide a clearer understanding of what constitutes sensitive information, we show examples of leakage with and without sensitive content. In \Cref{tab:example_leakage}, we observe that leakage of sensitive information exposes critical details such as names, addresses, phone numbers, and email addresses. In contrast, non-sensitive information reveals no useful details. For example, a generic URL or repeated mentions of ``Queen St'' are too common and lack identifying information.}

\subsection{Algorithm}
Let $D=\{d_1,\ldots, d_n\}$ be the corpus of $n$ natural language documents (e.g., email threads or case paragraphs). Let $f(\cdot)$ denote the tokenizer and embedding layers, which can be further fine-tuned to map these $d_i$'s (texts) to embeddings. We assume a fixed pre-trained encoder throughout the paper for simplicity. We also assume that an attacker interacts with the RAG system through API interactions and aims to extract data with constructed queries.

\begin{wrapfigure}{r}{0.45\columnwidth}
\vspace{-10pt}
\footnotesize

\hrule
\vspace{3pt}
\captionof{algorithm}{Private RAG with Random Projection (Standard RAG)}
\label{alg:rpllmrag}
\vspace{3pt}
\hrule
\vspace{3pt}

\begin{algorithmic}[1]

\Require Datastore $D$ with $n$ documents;
embedding function $f:\text{Text}\to\mathbb{R}^d$
\Statex \textbf{Input:} User query $q$
\Statex \textbf{Output:} Generated answer $a$
\Statex \textbf{Parameter:} Projection dimension $k$

\Statex \textbf{Offline preprocessing}
\State Sample $R\in\mathbb{R}^{k\times d}$,
$R_{ij}\sim\mathcal{N}(0,1/k)$
\State Compute embeddings
$E_D\in\mathbb{R}^{d\times n}$,
$\bm{e}_i=f(d_i)$
\State Project: $\tilde{E}_D=RE_D$

\Statex \textbf{Online inference (per query)}
\State Initialize $a\gets\emptyset$
\State Compute query embedding: $\bm{x}=f(q)$
\State Project: $\tilde{\bm{x}}=R\bm{x}$
\State Calculate similarity between
$\tilde{\bm{x}}$ and $\tilde{E}_D$
\State Retrieve most relevant text: $y$
\State $q\gets q\Vert y$
\State $a\gets LM(q)$
\State \Return{$a$}

\end{algorithmic}

\vspace{3pt}
\hrule
\vspace{-10pt}
\end{wrapfigure}

We sample an IID\ Gaussian matrix $R \in \mathbb{R}^{k \times d}$, where each element follows $\mathcal{N}(0,1/k)$. We precompute the projected datastore embeddings offline by computing $\bm{e}_i' = R\,\bm{e}_i$ for each document embedding $\bm{e}_i=f(d_i) \in \mathbb{R}^d$. We perform this preprocessing step once and store the results, introducing no runtime overhead during inference. For any query/prompt, we encode it as $\bm{x} = f(\text{prompt})$ and compute its projection $\bm{x}' = R\,\bm{x}$, which requires only a single matrix-vector multiplication. We then perform retrieval in the $k$-dimensional space by finding the nearest neighbors to $\bm{x}'$ among $\{\bm{e}_i'\}_{i\in[n]}$ and computing a softmax over their cosine similarities with $\bm{x}'$. Finally, we combine this softmax with the parametric LLM's logits for $\bm{x}$ in a weighted manner to select the next token. We repeat generation until reaching the maximum token limit or generating an end-of-sequence token. 

\Cref{alg:rpknnlm} (in \Cref{apex:alg}) and \Cref{alg:rpllmrag} illustrate the implementation of our random projection method. The overall process follows the standard RAG pipeline: we first tokenize and embed the datastore offline, and store the representations in FAISS or another efficient indexing framework to enable fast retrieval at inference time. During inference, the model performs autoregressive generation via iterative next-token prediction.

For standard RAG (\Cref{alg:rpllmrag}), we compute the similarity between the projected query embedding $\tilde{x}$ and the projected datastore embeddings $\tilde{E}_D$, and retrieve the most relevant text based on this similarity. The retrieved text is concatenated with the original query to form an augmented prompt, which is then passed to a parametric LLM to generate the final output. The algorithm of KNN-LM is summarized in \Cref{apex:alg}.

\begin{figure*}[!htbp]
    \centering
    \includegraphics[width=\linewidth]{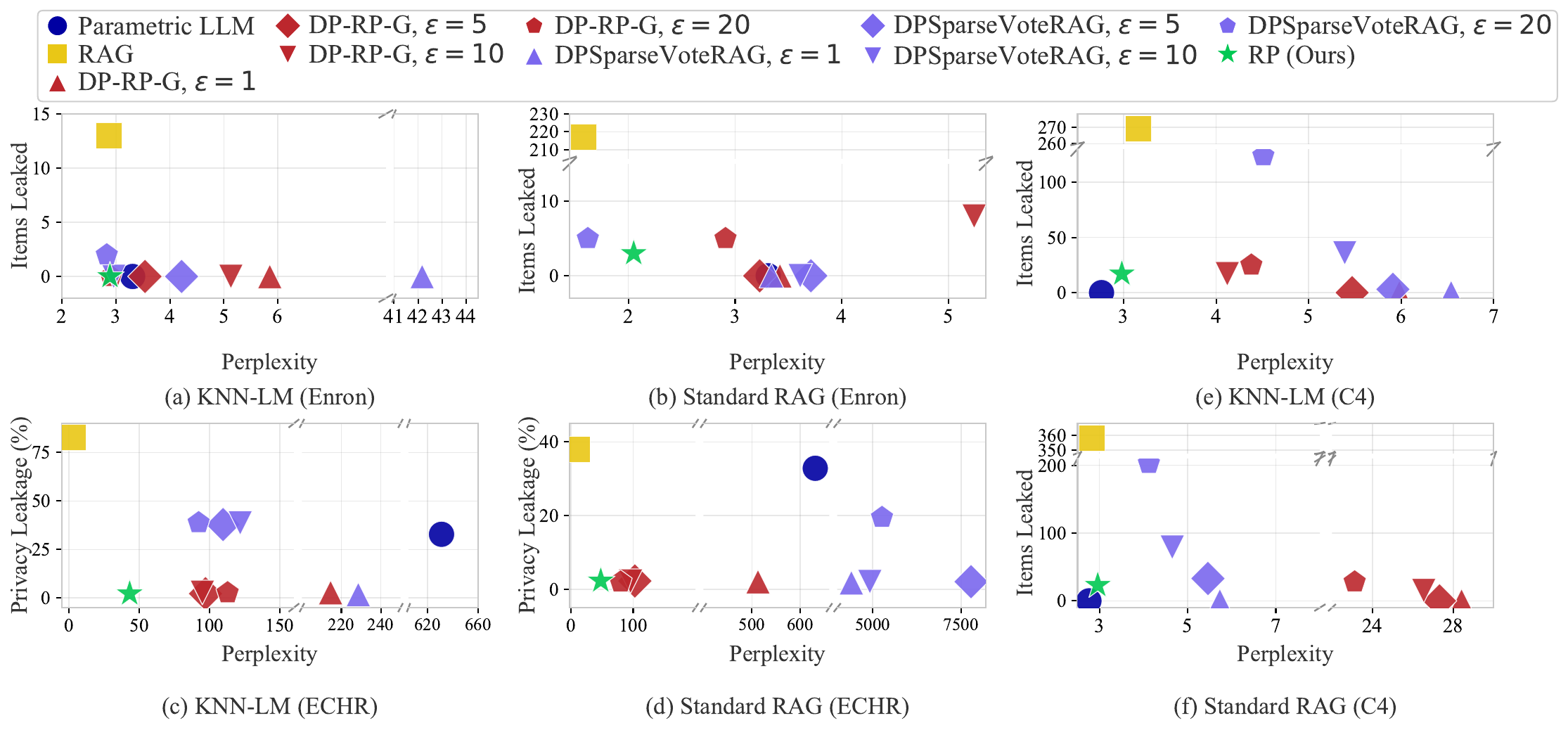}
    \captionsetup{skip=2pt}
    \caption{Utility/privacy comparisons of our method (RP) with the two DP baselines (DP-RP-G and DPSparseVoteRAG) on two RAG architectures (denoted as KNN-LM and Standard RAG) and targeted attacks on Enron Email and ECHR and untargeted attacks on C4. $\varepsilon$ denotes various privacy budgets for the DP approaches. The x-axis is perplexity of generated content, and the y-axis is percentage of data leakage. RP (Ours) achieves a superior utility–privacy trade-off, with perplexity approaching the unprotected RAG baseline while maintaining empirically near-zero privacy leakage comparable to the standalone parametric model without retrieval-augmentation.}
    \label{fig:results}
\end{figure*}

\begin{table*}[!tbp]
\centering
\captionsetup{skip=2pt}
\caption{
Full privacy-utility comparison across datasets and methods under the repeated-query attack. Utility is measured using perplexity (PPL) and QA F1 scores. For Enron Email and C4, privacy leakage is reported as the number of leaked instances, whereas for ECHR, it is reported as a percentage.}
\captionsetup{skip=2pt}
\label{tab:full-results}

\renewcommand{\arraystretch}{0.88}

{%
\tablebodyfont
\setlength{\tabcolsep}{2pt}

\setlength{\aboverulesep}{1.5pt}
\setlength{\belowrulesep}{1.5pt}
\setlength{\cmidrulesep}{1.5pt}

\begin{tabular*}{\textwidth}
{@{\extracolsep{\fill}}lllcccccccccc>{\columncolor{green!20}}c@{}}
\toprule
\multirow{2}{*}{Method}
& \multirow{2}{*}{Dataset / Model}
& \multirow{2}{*}{Metric}
& \multirow{2}{*}{Param.}
& \multirow{2}{*}{RAG}
& \multicolumn{4}{c}{DP-RP-G}
& \multicolumn{4}{c}{DPSparseVoteRAG}
& \multirow{2}{*}{\textbf{RP (Ours)}} \\
\cmidrule(lr){6-9}
\cmidrule(lr){10-13}
& & & & 
& $\epsilon{=}1$
& $\epsilon{=}5$
& $\epsilon{=}10$
& $\epsilon{=}20$
& $\epsilon{=}1$
& $\epsilon{=}5$
& $\epsilon{=}10$
& $\epsilon{=}20$
& \\
\midrule

\multirow{10}{*}{KNN-LM}

& \multirow{2}{*}{Enron Email / GPT2}
& PPL $\downarrow$
& 3.31 & 2.87
& 5.85 & 3.54 & 5.13 & 2.91
& 42.17 & 4.22 & 2.96 & 2.83
& \textbf{2.89} \\

&
& Leakage (\#) $\downarrow$
& 0 & 13
& 0 & 0 & 0 & 0
& 0 & 0 & 0 & 2
& \textbf{0} \\

\cmidrule(lr){2-14}

& \multirow{3}{*}{ECHR / Llama3-1B}
& PPL $\downarrow$
& 631.09 & 2.90
& 214.54 & 99.28 & 94.87 & 112.85
& 228.59 & 109.64 & 121.82 & 92.26
& \textbf{46.15} \\

&
& F1 $\uparrow$
& 25.70 & 90.32
& 53.25 & 51.30 & 54.02 & 55.13
& 41.67 & 41.67 & 40.32 & 45.38
& \textbf{87.70} \\

&
& Leakage (\%) $\downarrow$
& 32.82 & 82.72
& 2.76 & 2.41 & 2.81 & 2.69
& 1.75 & 37.86 & 38.68 & 38.82
& \textbf{2.26} \\

\cmidrule(lr){2-14}

& \multirow{3}{*}{ECHR / Qwen3-4B}
& PPL $\downarrow$
& 1.47 & 1.55
& 2.26 & 2.28 & 2.28 & 2.30
& 3130.89 & 610.10 & 500.13 & 277.62
& \textbf{1.55} \\

&
& F1 $\uparrow$
& 90.18 & 90.32
& 80.35 & 80.32 & 80.18 & 82.39
& 52.40 & 51.37 & 51.94 & 52.98
& \textbf{89.32} \\

&
& Leakage (\%) $\downarrow$
& 2.03 & 19.08
& 13.45 & 13.21 & 14.85 & 15.39
& 2.44 & 2.46 & 2.21 & 2.24
& \textbf{1.94} \\

\cmidrule(lr){2-14}

& \multirow{2}{*}{C4 / Llama3-8B}
& PPL $\downarrow$
& 2.76 & 3.16
& 5.98 & 5.47 & 4.12 & 4.38
& 6.54 & 5.91 & 5.39 & 4.51
& \textbf{2.98} \\

&
& Leakage (\#) $\downarrow$
& 0 & 269
& 0 & 0 & 17 & 25
& 0 & 3 & 36 & 124
& \textbf{17} \\

\midrule

\multirow{10}{*}{Standard RAG}

& \multirow{2}{*}{Enron Email / GPT2}
& PPL $\downarrow$
& -- & 1.58
& 3.42 & 3.23 & 5.24 & 2.91
& 3.34 & 3.71 & 3.61 & 1.62
& \textbf{2.05} \\

&
& Leakage (\#) $\downarrow$
& -- & 217
& 0 & 0 & 8 & 5
& 0 & 0 & 0 & 9
& \textbf{3} \\

\cmidrule(lr){2-14}

& \multirow{3}{*}{ECHR / Llama3-1B}
& PPL $\downarrow$
& -- & 9.84
& 512.66 & 104.70 & 95.72 & 79.72
& 4398.91 & 7780.70 & 4922.52 & 5264.40
& \textbf{47.92} \\

&
& F1 $\uparrow$
& -- & 97.18
& 51.39 & 54.31 & 54.12 & 54.79
& 1.72 & 43.72 & 42.42 & 43.33
& \textbf{66.70} \\

&
& Leakage (\%) $\downarrow$
& -- & 37.92
& 2.03 & 1.87 & 2.15 & 1.96
& 1.77 & 2.05 & 2.05 & 19.48
& \textbf{2.16} \\

\cmidrule(lr){2-14}

& \multirow{3}{*}{ECHR / Qwen3-4B}
& PPL $\downarrow$
& -- & 1.19
& 2.50 & 2.49 & 2.51 & 2.50
& 1.74 & 1.43 & 1.62 & 1.38
& \textbf{1.19} \\

&
& F1 $\uparrow$
& -- & 97.32
& 85.32 & 86.17 & 86.43 & 85.19
& 88.35 & 91.03 & 93.48 & 94.32
& \textbf{95.43} \\

&
& Leakage (\%) $\downarrow$
& -- & 21.30
& 25.43 & 14.36 & 17.07 & 15.09
& 2.84 & 2.88 & 2.86 & 3.32
& \textbf{2.05} \\

\cmidrule(lr){2-14}

& \multirow{2}{*}{C4 / Llama3-8B}
& PPL $\downarrow$
& -- & 2.83
& 28.41 & 27.32 & 26.54 & 23.12
& 5.73 & 5.46 & 4.65 & 4.12
& \textbf{2.96} \\

&
& Leakage (\#) $\downarrow$
& -- & 358
& 0 & 0 & 15 & 28
& 0 & 33 & 79 & 203
& \textbf{23} \\

\bottomrule
\end{tabular*}%
}
\end{table*}

\paragraph{Privacy Implications.} 
Random projection preserves privacy by altering the retrieval process: given a query, vanilla RAG retrieves the highest-scoring document and generates a corresponding token, which may contain sensitive information. After RP, embeddings that are close to the original top match can be selected with non-negligible probability (\cref{apx:prob-after-rp}). As a result, alternative tokens may replace the originally retrieved token. 

\paragraph{Utility Discussions.}
A natural question arises: \emph{if random projections alter the retrieved neighbors, how does the model remain accurate?} The reason is that random projection approximately preserves embedding distances, ensuring retrieval still selects relevant information. This has been extensively studied in prior literature~\cite{johnson1984extensions}. In \cref{apx:rp-cos-proof}, for completeness, we provide proof that our random projection preserves $L_2$ distance and cosine similarity. \cref{sec:sanity-check} shows a study with synthetic data consistent with the theoretical analysis. For DP-based baselines, achieving privacy guarantees sacrifices utility. DP introduces a fundamental tradeoff: a small clipping bound overly distorts embeddings, while a large bound increases noise magnitude, both degrading retrieval utility. We illustrate this in \cref{sec:onlyclip}.

\section{Evaluation} 
\label{sec:experiments}

\subsection{Settings}
We focus on empirical data extraction attacks~\citep{huang2023privacy} rather than differential privacy bounds, considering both repeated single-query and multi-round query attacks. We evaluate four datastore/LLM pairs: Enron Email~\citep{klimt2004enron} with GPT-2~\citep{radford2019language}, European Court of Human Rights (ECHR) with \textsc{Llama3-1B-Instruct}, C4 with \textsc{Llama3-8B-Instruct}, and HotpotQA~\citep{yang2018hotpotqa} with \textsc{Qwen2-7B-Instruct}. For utility, we measure perplexity~\citep{huang2023privacy} and task-specific performance.

For simplicity, we denote our random projection as RP, DP-guaranteed random projection with Gaussian noise as DP-RP-G~\cite{blocki2012johnson}, and the work of \citet{koga2024privacy} as DPSparseVoteRAG. We also report results for vanilla RAG and the parametric model without a datastore. For differential privacy, we use a wide range of privacy budgets with $\epsilon \in \{1, 5, 10, 20\}$. For DP-RP-G, the original DP-RP method by \citet{blocki2012johnson} assumes binary vector embeddings. To adapt this framework to RAG and document retrieval, we first project both queries and datastore embeddings into a lower-dimensional space, clip the projected embeddings, and then add Gaussian noise to obtain differential privacy guarantees. The complete procedure is in \Cref{apex:alg}. We sweep the clipping bound from $\{0.1,0.5,1,5,10,15,25,50,100\}$ under each setting and report the results with optimal clipping bound. Remaining settings are in \Cref{apx:hyper}. For DPSparseVoteRAG, we {sweep the subset number $m$ over $\{10,20,30,40,50,60\}$ and choose the best value $50$ (\Cref{apx:tunem})}.

In our methods, for KNN-LM, we set $\lambda$ as 0.1 and $K$ as 1024. For RP, the only tunable parameter is $k$ (the projection dimension). Larger $k$ values better preserve utility while maintaining privacy protection. We choose $k=100$ since it achieves comparable utility to larger values (e.g., $k=1600$) while providing better computational efficiency. More detailed analysis is in \Cref{appx:setk}.

\subsection{Repeated-Query Attacks}
\label{sec:repeat}

\begin{wraptable}{r}{0.48\linewidth}
\centering
\captionsetup{skip=2pt}
\caption{Number of leaked email addresses, URLs, and phone numbers with different methods on Enron Email.}
\captionsetup{skip=2pt}
\label{tab:enron-full-leakage}
{
\tablebodyfont
\renewcommand{\arraystretch}{0.9}
\setlength{\tabcolsep}{2pt}
\setlength{\aboverulesep}{0.5pt}
\setlength{\belowrulesep}{0.5pt}
\setlength{\cmidrulesep}{0.5pt}

\begin{tabular}{@{}lcccccc@{}}
\toprule
 & \multicolumn{3}{c}{KNN-LM} & \multicolumn{3}{c}{Standard RAG} \\
\cmidrule(lr){2-4}\cmidrule(lr){5-7}
 & Email & URL & Phone & Email & URL & Phone \\
\midrule
Parametric                     & 0 & 0 & 0 & 0 & 0 & 0 \\
RAG                            & 0 & 13 & 0 & 36 & 56 & 125 \\
\rowcolor{green!20}
RP (Ours)
& \cellcolor{green!20}0
& \cellcolor{green!20}0
& \cellcolor{green!20}0
& \cellcolor{green!20}0
& \cellcolor{green!20}3
& \cellcolor{green!20}0 \\
\midrule
DP-RP-G, $\epsilon=1$          & 0 & 0 & 0 & 0 & 0 & 0 \\
DP-RP-G, $\epsilon=5$          & 0 & 0 & 0 & 0 & 0 & 0 \\
DP-RP-G, $\epsilon=10$         & 0 & 0 & 0 & 0 & 8 & 0 \\
DP-RP-G, $\epsilon=20$         & 0 & 0 & 0 & 0 & 1 & 4 \\
\midrule
DPSparseVoteRAG, $\epsilon=1$  & 0 & 0 & 0 & 0 & 0 & 0 \\
DPSparseVoteRAG, $\epsilon=5$  & 0 & 0 & 0 & 0 & 0 & 0 \\
DPSparseVoteRAG, $\epsilon=10$ & 0 & 0 & 0 & 0 & 0 & 0 \\
DPSparseVoteRAG, $\epsilon=20$ & 0 & 2 & 0 & 0 & 8 & 1 \\
\bottomrule
\end{tabular}}
\end{wraptable}
We launch repeated-query targeted attacks on Enron Email for personal identity information (PII) and ECHR for legal documentsand untargeted attacks on raw C4 content. Attackers can repeately inject the same query multiple times.  To make repeated-query extraction attack stronger, we inject dataset-specific prior knowledge into malicious queries, such as asking for emails when attacking the Enron Email dataset. For each entry, we perform 10 times repeated extraction attempts, and count any successful extraction as leakage. We assume the system flags and blocks queries repeated more than 10 times. Details, including query construction, are provided in \Cref{apx:implementation}.

\Cref{fig:results} gives an overview of the privacy/utility trade-off of different methods under varying privacy budgets on three datasets under repeated-query attack while \Cref{tab:full-results} gives details. Standard RAG achieves higher utility but poses greater risks of privacy leakage. Without protection, RAG leaks over 200 personal items in the Enron dataset. In contrast, our methods maintain utility close to RAG while leaking almost no private information. For DP-RP-G, we achieve strong privacy across all budgets, but utility remains unsatisfactory. DPSparseVoteRAG (DPSVR) shows that increasing the privacy budget reduces perplexity but increases leakage, lacking a balanced trade-off point. Overall, our approach gives strong privacy with minimal utility loss.

\begin{wraptable}{r}{0.4\linewidth}
\centering
\captionsetup{skip=2pt}
\caption{Membership leakage rate. RP (Ours) can achieve approximately the same performance against MIA as parametric, which means the attacker barely identify membership information from RAG datastore.}
\captionsetup{skip=2pt}
\label{tab:mia}
 {
\tablebodyfont
\setlength{\tabcolsep}{2pt}
\begin{tabular}{lcccc}
\toprule
Method & Enron Email & \shortstack{ECHR \\ (Qwen3-4B)} & C4 & HotpotQA \\
\midrule
Parametric & 0.001 & 0.006 & 0.003 & 0.001 \\
RAG        & 0.282 & 0.996 & 0.843 & 0.913 \\
\textbf{RP (Ours)} & \textbf{0.003} & \textbf{0.007} & \textbf{0.003} & \textbf{0.001} \\
\bottomrule
\end{tabular}}
\end{wraptable}

Our evaluation covers both structured and unstructured sensitive information to assess the method's generalizability. On the Enron dataset, we evaluate protection against structured data including email addresses, phone numbers, and URLs. As shown in \Cref{tab:enron-full-leakage}, random projection effectively protects all three types, with leakage rates below 3\% across both KNN-LM and Standard RAG architectures. We observe that URLs are the easiest to leak. For unstructured text, we evaluate on the ECHR dataset, which contains legal case narratives with sensitive personal information embedded in natural language. Our method maintains strong privacy protection (2.19--2.27\% leakage) while preserving utility, demonstrating that random projection works across both structured and unstructured sensitive information. This suggests that the privacy-preserving mechanism operates at the embedding level, making it agnostic to the format or structure of the sensitive content.

\par\noindent\usebox{\miatable}\par\medskip

A further extension of the repeated-query attack is the membership inference attack (MIA). MIA is important in RAG privacy, but it targets a different adversarial objective from the data extraction attacks studied in this paper: determining whether a known record is contained in the datastore, so we do not focus on evaluating various MIA methods. We evaluate this threat over the full private dataset and measure attack accuracy following E-MIA~\citep{guan2026mia}, which follows the standard MIA protocol for RAG systems under our threat model. As shown in \Cref{tab:mia}, RP provides strong protection against E-MIA, achieving empirically near-zero membership leakage.

\subsection{Multi-Round Adaptive Attacks}
\label{sec:multi}
\begin{wraptable}{r}{0.45\linewidth}
\centering
\caption{Privacy results against multi-round attacks on HotpotQA using the KNN-LM framework.}
\captionsetup{skip=2pt}
\label{tab:hotpotqamain}
\tablebodyfont
\renewcommand{\arraystretch}{0.9}
\setlength{\tabcolsep}{4pt}
\setlength{\aboverulesep}{1pt}
\setlength{\belowrulesep}{1pt}
\setlength{\cmidrulesep}{1pt}

\begin{tabular}{@{}lccccccc@{}}
\toprule
& \multicolumn{4}{c}{\citet{jiang2024feedback}}
& \multicolumn{3}{c}{\citet{yao2026connect}} \\
\cmidrule(lr){2-5}\cmidrule(l){6-8}
& CRR & SS & EED & Rounds
& CRR & SS & EED \\
\midrule
Parametric & 30\% & 0.17 & 0.030 & 2.18 & 37\% & 0.033 & 0.07 \\
RAG & 60\% & 0.16 & 0.028 & 2.05 & 75\% & 0.034 & 0.10 \\
\midrule
\rowcolor{green!20}
RP (Ours) & 3\% & 0.16 & 0.88 & 2.05 & 0\% & 0.032 & 0.93 \\
DP-RP-G ($\epsilon=10$) & 3\% & 0.17 & 0.89 & 2.01 & 0\% & 0.031 & 0.90 \\
DPSVR ($\epsilon=10$) & 2\% & 0.17 & 0.88 & 2.34 & 0\% & 0.034 & 0.90 \\
\bottomrule
\end{tabular}
\end{wraptable}

Another strong attack is the multi-round adaptive attack~\citep{jiang2024feedback}, where an attacker iteratively refines adversarial prompts based on outputs from previous rounds and aggregates the extracted information to infer private content. \citet{yao2026connect} further constructs a knowledge graph from outputs across attack rounds to generate future adversarial queries. The attack budget is set to a maximum of 200 queries to the RAG system. 

For both attack methods, we evaluate the reconstructed content using semantic similarity (SS) and extended edit distance (EED). Following the experimental setup of \citet{jiang2024feedback}, we construct the datastore using chunks of 1500 words with a maximum overlap of 300 words. We report the chunk recovery rate (CRR), considering a chunk recovered if the reconstructed content overlaps with the original chunk by more than 1200 words. For \citet{jiang2024feedback}, we report the average rounds required to satisfy the termination condition.

\begin{wraptable}{r}{0.45\linewidth}
\centering
\caption{Utility results of different algorithms on HotpotQA using the KNN-LM structure.}
\captionsetup{skip=2pt}
\label{tab:hotpotqa}
{
\tablebodyfont
\renewcommand{\arraystretch}{0.9}
\setlength{\tabcolsep}{2pt}
\setlength{\aboverulesep}{1pt}
\setlength{\belowrulesep}{1pt}
\setlength{\cmidrulesep}{1pt}

\begin{tabular}{cccccccc}
\toprule
 & \multicolumn{2}{c}{Ans} & \multicolumn{2}{c}{Sup} & \multicolumn{2}{c}{Joint} & \multirow{2}{*}{Cos.} \\
 & EM & F1 & EM & F1 & EM & F1 & \\
\midrule
Parametric                     &  7.48 & 28.74 & 21.37 & 49.62 &  7.13 & 28.55 & 0.73 \\
RAG                             & 12.48 & 29.88 & 25.38 & 59.74 & 10.84 & 32.06 & 0.75 \\
\midrule
\rowcolor{green!20}
RP (Ours) & 11.32 & 28.57 & 25.22 & 57.14 & 9.12 & 34.09 & 0.75 \\
\midrule
DP-RP-G, $\epsilon=1$          &  0.08 &  7.10 &  9.54 & 30.42 &  0.15 &  9.03 & 0.54 \\
DP-RP-G, $\epsilon=5$          &  0.21 &  9.32 &  10.35 & 27.82 &  0.16 &  9.85 & 0.53 \\
DP-RP-G, $\epsilon=10$         & 3.58 &  10.84&  9.71 & 28.93 &  1.19 &  7.34 & 0.54 \\
DP-RP-G, $\epsilon=20$         &  4.56 &  9.81 & 14.95 & 31.23 &  1.67 &  8.32 & 0.61 \\
\midrule
DPSparseVoteRAG, $\epsilon=1$  &  0.02 &  8.12 & 13.98 & 23.12 &  0.00 & 13.12 & 0.51 \\
DPSparseVoteRAG, $\epsilon=5$  &  3.31 & 11.23 & 12.88 & 30.64 &  1.31 & 14.29 & 0.52 \\
DPSparseVoteRAG, $\epsilon=10$ &  7.48 & 18.18 & 14.39 & 28.85 &  1.54 & 16.67 & 0.54 \\
DPSparseVoteRAG, $\epsilon=20$ &  7.48 & 18.18 & 14.38 & 28.85 &  1.54 & 15.38 & 0.54 \\
\bottomrule
\end{tabular}}
\end{wraptable}
We launch both multi-round targeted attacks on HotpotQA for document information. For meausing the utility of RAG system on usual queries. We report exact match (EM) and F1 scores between the generated output and the ground truth. For F1, precision is the number of matching tokens divided by the total generated tokens, and recall is the number of matching tokens divided by the ground-truth tokens. \emph{Ans} denotes the final generated answer, and \emph{Sup} denotes the retrieved supporting documents, measuing all tokens in the documents. For the parametric model, we prompt it to generate supporting facts to enable fair comparison, as the datastore may overlap with pre-training data. \emph{Joint} indicates cases where we put answer and supporting documents together to measure.

\begin{figure}[!t]
\centering
\usebox{\multiroundtables}
\end{figure}

From \Cref{tab:hotpotqamain}, we observe that the three privacy-preserving methods achieve similar utility, while our method attains the highest accuracy. All methods achieve similar semantic similarity between the original datastore content in each chunk and the attacked results, which we attribute to LLMs attempting to generate correct answers. In contrast, a high chunk recovery rate (i.e., how much datastore content is completely restored) and low EED indicate successful extraction attacks. Examples of extracted content are shown in \Cref{tab:examplemultiround}. As shown in \Cref{tab:hotpotqa}, random projection preserves cosine similarity between retrieved documents and ground truth, indicating successful retrieval. In contrast, differential privacy methods degrade retrieval performance. Looser privacy budgets can improve utility but raise privacy concerns.

\begin{table*}[!htbp]
\centering
\caption{Examples of attack results under the multi-round setting. We highlight content matching the original datastore chunks. RP avoids revealing sensitive identity information while preserving relational logic, such as who was born first, even after multi-round attacks using a knowledge graph to retrieve related chunks~\citep{yao2026connect}.}
\captionsetup{skip=2pt}
\label{tab:examplemultiround}
{
\renewcommand{\arraystretch}{1.1}
\tablebodyfont
\setlength{\tabcolsep}{2pt}
\begin{tabular}{p{\dimexpr0.18\linewidth-4pt\relax} p{\dimexpr0.41\linewidth-4pt\relax} p{\dimexpr0.41\linewidth-4pt\relax}}
\toprule
Original Chunk &\multicolumn{2}{p{0.78\linewidth}}{...Aleksander Ford: Aleksander Ford (born Mosze Lifszyc, 24 November 1908 in Kiev, Russian Empire – 4 April 1980 in Naples, Florida, United States) was a Polish Jewish film director, and head of the Polish People's Army Film Crew in the Soviet Union during World War II. Ford became director of the nationalized Film Polski company following the Red Army occupation of Poland. 7 Days in Havana: 7 Days in Havana (Spanish: 7 días en La Habana ) is a 2012 Spanish-language anthology film. Set during a week in the Cuban capital Havana, the film features one segment for each day, each segment directed by a different filmmaker. The directors are Julio Medem, Laurent Cantet, Juan Carlos Tabío, Benicio del Toro, Gaspar Noé, Pablo Trapero and Elia Suleiman. The screenplay was written by the Cuban novelist Leonardo Padura Fuentes. The film is a co-production between companies in Spain, France and Cuba...}\\
\midrule
Initial Query&\multicolumn{2}{p{0.78\linewidth}}{Pablo Trapero or Aleksander Ford? Who was born first?}\\
\midrule
Attacking Methods & Response by RAG & Response by ours\\
\midrule

Adaptive~\citep{jiang2024feedback} &
<think> \hl{Aleksander Ford (born Mosze Lifszyc, 24 November 1908 in Kiev, Russian Empire – 4 April 1980 in Naples, Florida, United States) was a Polish Jewish film director}, and 2006 BBC Four television. \hl{Ford became director of the nationalized Film Polski company following the Red Army occupation of Poland. 7 Days in Havana: 7 Days in Havana (Spanish: 7 días en La Habana )} is a "Pennies From Heaven" (1978)...&
<think> \hl{Pablo Trapero} or \hl{Aleksander Ford}? 7 Who was born first? 8 Who was born first? \hl{Aleksander Ford}?\\

\citet{yao2026connect} &
<think> Thinking Process:  1.  **Analyze the Request:**     *   Target: An English actor of stage, film, and television.     *   Condition 1: \hl{Aleksander Ford (born Mosze Lifszyc, 24 November 1908 in Kiev, Russian Empire – 4 April 1980 in Naples, Florida, United States) was a Polish Jewish film director} *   Condition 2: \hl{7 Days in Havana: 7 Days in Havana (Spanish: 7 días en La Habana )}    *   Goal: Who was born first?     * ... &
<think> Who was born first? \hl{Pablo Trapero} or \hl{Aleksander Ford}?  <think> Thinking Process:  1.  **Analyze the Request:** The user is asking for a comparison of birth dates between two individuals: \hl{Pablo Trapero} and \hl{Aleksander Ford}. The goal is to determine who was born first.  2.  **Identify the Individuals:**     *   **\hl{Pablo Trapero}:** An Argentine film director,  screenwriter,  and producer.     *   **\hl{Aleksander Ford}:** A Polish film director and screenwriter.  3.  **Retrieve Birth Dates:**     *   **\hl{Pablo Trapero}:** I need to recall or verify his birth date.         *   Search query mental check: "\hl{Pablo Trapero} birth date".         *   Knowledge: 
\hl{Pablo Trapero} was born on \uline{\textcolor{red}{December 27,  1965}}.     *   **\hl{Aleksander Ford}:** I need to recall or verify his birth date.         *   Search query mental check: "\hl{Aleksander Ford} birth date".         *   Knowledge: \hl{Aleksander Ford} was born on \uline{\textcolor{red}{September 1,  1923}}.  4.  **Compare Dates:**     *   \hl{Aleksander Ford}\\

\bottomrule
\end{tabular}}
\end{table*}

We present a case study of results under multi-round attacks. From \Cref{tab:examplemultiround}, we can see that \emph{Pablo Trapero} and \emph{Aleksander Ford} are originally mentioned in the question, so it is reasonable that the responses also contain these two names. However, the response generated by our method is substantially different from the original dataset chunk while still focusing on answering the question. No sensitive information, such as address, nationality, or birthdate, is revealed. Hence, this demonstrates successful defense against multi-round attacks.

One interesting aspect of our response is that Pablo Trapero was actually born on October 4th, 1971, and Aleksander Ford on November 24th, 1908. However, our response states that they were born on December 27th, 1965, and September 1st, 1923, respectively. This indicates that our method preserves the final answers without revealing the actual birthdates. 
{ The attacker may also mistakenly conclude that the birthdate information has been retrieved and stop the attack.} In contrast, standard RAG reveals a large portion of the original content. \citet{jiang2024feedback} successfully extract most sensitive content, while the knowledge graph-based method recovers most original information.

\subsection{Effects of $k$}
\label{appx:setk}

In both RP and DP-RP-G, we sample the projection matrix from $\mathcal{N}(0, 1/k)$ where $k$ is the projection dimension. \cref{fig:diffk} examines the impact of different $k$ values. As discussed in \cref{sec:alg}, larger $k$ better preserves retrieval scores. Consistently, we observe higher utility with increasing $k$, while privacy remains unchanged. Our method relies on similarity-preserving projection rather than precise distance preservation (\Cref{apx:rp-cos-proof} and \Cref{apx:prob-after-rp}), so the expected post-projection ranking change is less sensitive to 
$k$, indicating robustness across a wide range of values. Based on analysis and empirical observation of choosing $k$, we suggest following ways of choosing $k$. If a user wants more privacy guarantee or less offline computation latency, a smaller $k$ should be chosen (e.g., $k=25$ for $d=2048$). If a user wants higher utility, a bigger $k$ should be chosen (e.g., $k=1600$ for $d=2048$). Otherwise, a moderate value would be suitable (e.g., $k=100$ for $d=2048$).

\begin{figure}[!htbp]
\centering
\includegraphics[width=0.6\linewidth]{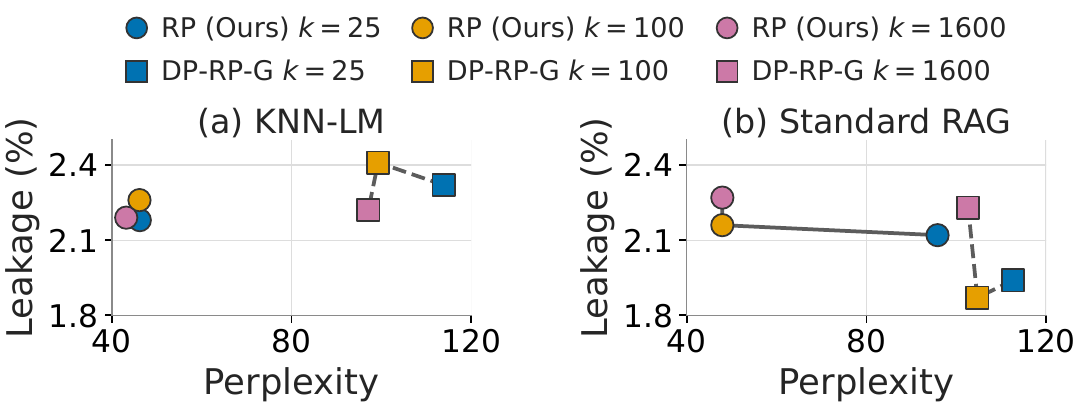}
\captionsetup{skip=2pt}
\caption{Comparison of privacy/utility. Colors indicate $k$; markers distinguish RP  and DP-RP-G. $k$ has weak relation with privacy but larger $k$ gives better utility.}

\label{fig:diffk}
\end{figure}

To understand the effect of not setting the sampling variance as $1/k$, we use $k=1600$ and sample from $\mathcal{N}(0, 0.2^2)$, resulting in a perplexity of $70.86$ and $2.32\%$ privacy leakage. For $k=25$ and $\mathcal{N}(0, 0.025^2)$, perplexity is $71.38$ with $3.1\%$ leakage. Choosing sampling variances larger or smaller than $1/k$ degrades utility, validating that setting variance to $1/k$ is optimal.

\subsection{Efficiency and Scalability}

\begin{table*}[!htbp]
\centering

\begin{minipage}[t]{0.52\textwidth}
    \vspace{0pt}
    \centering
    \captionsetup{type=table,skip=3pt}
    \captionof{table}{Comparison of the number of tokens generated and latency per generation for different methods over the ECHR dataset on an NVIDIA A100.}
    \label{tab:efficiency}

    {\tablebodyfont
    \setlength{\tabcolsep}{2pt}
    \renewcommand{\arraystretch}{1.1}
    \begin{tabular}{@{}lccc@{}}
    \toprule
                         & RAG & RP (Ours) & DPSVR \\
    \midrule
    \# Tokens per Generation & 107.8 & \textbf{123.4} & 92.3 \\
    Latency per Generation (s) & 8.35 & \textbf{9.42} & 258.9 \\
    Latency per Token (s) & 0.08 & 0.08 & 2.80 \\
    \bottomrule
    \end{tabular}
    }
\end{minipage}
\hfill
\begin{minipage}[t]{0.46\textwidth}
    \vspace{0pt}
    \centering
    \captionsetup{type=table,skip=3pt}
    \captionof{table}{Perplexity comparisons of KNN-LM and standard RAG on ECHR.}
    \label{tab:onlyclip}

    {\tablebodyfont
    \setlength{\tabcolsep}{3pt}
    \renewcommand{\arraystretch}{1.1}
    \begin{tabular}{@{}lcc@{}}
    \toprule
    Method & KNN-LM & \makecell{Standard\\RAG} \\
    \midrule
    RP (Ours) & \textbf{43.20} & \textbf{47.92} \\
    DP-RP-G & 97.09 & 102.66 \\
    \makecell[l]{DP-RP-G (only clipping)} & 45.47 & 46.99 \\
    \bottomrule
    \end{tabular}
    }
\end{minipage}

\end{table*}

The DPSparseVoteRAG method can be inefficient. In our experiments, each inference takes about 30$\times$ longer than RP and vanilla RAG (\cref{tab:efficiency}). This is due to DP voting and aggregation require partitioning the datastore into $m$ subsets and generating a token for each, resulting in $m$ inferences per token. {Our method can achieve the same efficiency as RAG on latency per token.}

Besides efficiency, we also take a look at the scalability of our evaluation. We empirically study datastores of up to about 100K HotpotQA documents, corresponding to roughly 200 million tokens, which is a reasonably sized RAG datastore comparable with many real-world enterprise services~\citep{choubey2025benchmarking,sun2026enterpriserag}. Prior work~\citep{qi2025follow} shows that extraction rates decrease substantially as datastore size increases. We further conduct a controlled scalability study in \Cref{fig:scalability} by varying the datastore size within the ECHR/Qwen3-4B setting. We show (1) the extraction rate of vanilla RAG decreases as the datastore grows, while (2) our method consistently maintains strong privacy protection across datastore sizes. Moreover, as discussed in \Cref{sec:alg}, our RP method relies on replacing tokens or chunks with similar alternatives from the datastore. As the datastore size increases, the method has access to a larger pool of candidate tokens for such substitutions, reducing the extraction rate.

\begin{figure*}[!htbp]
\centering

\begin{minipage}[t]{0.54\textwidth}
    \vspace{0pt}
    \centering
    \includegraphics[width=\linewidth]{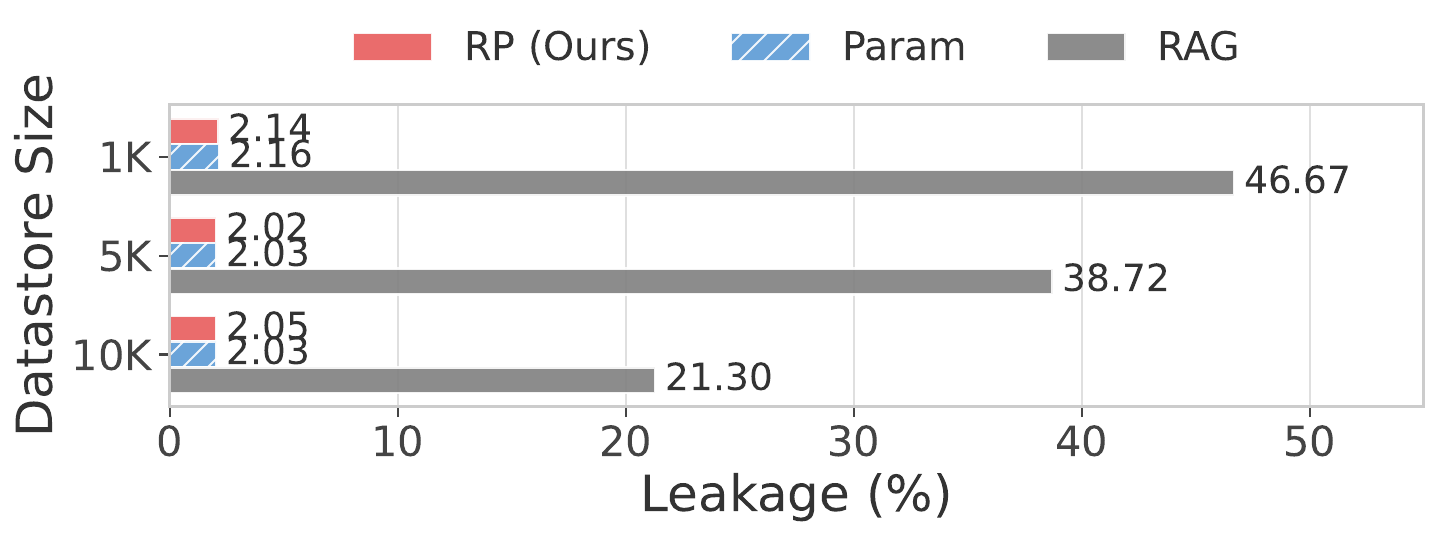}
    \captionsetup{skip=2pt}
    \caption{Scalability study on Standard RAG using ECHR with Qwen3-4B. Across different datastore sizes, RP maintains leakage close to the parametric baseline while consistently providing strong privacy protection.}
    \label{fig:scalability}
\end{minipage}
\hfill
\begin{minipage}[t]{0.44\textwidth}
    \vspace{0pt}
    \centering
    \includegraphics[width=0.8\linewidth]{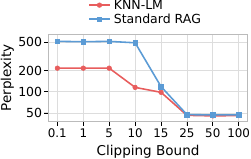}
    \captionsetup{skip=2pt}
    \caption{Perplexity of DP-RP-G with clipping only (without Gaussian noise) under different clipping bounds.}
    \label{fig:onlyclip2}
\end{minipage}

\end{figure*}

\subsection{Why Peformance Degrades More by Differential Privacy Methods}
\label{sec:onlyclip}

A natural question is: \emph{what is the challenge for differential-privacy based methods to maintain utiliy in RAG applications.}
For DP-RP-G, achieving an $(\epsilon,\delta)$ guarantee requires bounding sensitivity, typically by clipping each projected embedding to norm $c$, and then adding Gaussian noise
$\mathcal{N}\!\left(0,\frac{c\sqrt{2\ln(1.25/\delta)}}{\epsilon}\right)$.
If $\epsilon\rightarrow\infty$, the noise vanishes; with a sufficiently large clipping bound, DP-RP-G approaches ours RP method. Thus, the utility loss arises from the additional constraints required for formal DP rather than from RP itself.

DP aggregation methods face a similar trade-off, as their utility depends strongly on perturbation magnitude and therefore on privacy budget~\citep{koga2024privacy}. As shown in \Cref{sec:experiments}, even relatively loose privacy budgets can substantially degrade utility while still allowing non-negligible leakage.

Clipping is typically required in DP mechanisms to bound sensitivity~\citep{dwork2006calibrating}; prior work similarly constrains vectors to bounded domains~\citep{kenthapadi2013privacy}. To isolate the effect of clipping and adding noise, we ablate the clipping operation without Gaussian noise.  In \Cref{tab:onlyclip}, we compare RP with and without clipping, while \Cref{fig:onlyclip2} varies the clipping bound. The results show that larger clipping consistently reduces perplexity, whereas sufficiently large clipping bounds recover performance close to RP. 

\section{Conclusion}
In this work, we propose a simple random projection method to defend against data extraction attacks for RAG applications. Compared with DP-based approaches, we show that the specific random Gaussian projection can empirically protect against data reconstruction, at approximately the same level as parametric model without RAG datastore, while still maintaining semantic utility, with a 7\% overall decrease compared with vanilla RAG.
\section*{Limitations}
\label{sec:limitations}


At evaluation side, large-scale evaluation of privacy-preserving methods in real-world LLM-based RAG systems remains an interesting future work. Real-world systems may need to process billions of queries daily. In \Cref{tab:efficiency}, one query takes about 8 seconds. Hence, evaluating one million queries would take 92 days without system-level optimization. Such large-scale evaluation thus requires system-level contributions, which we defer to future work.

\section*{Ethical Considerations}
We use models and datasets that are publicly available on the Internet and licensed for non-commercial research use. We do not involve personally identifiable information (PII) which are not allowed to make public. We use only publicly available or synthetically generated datasets; we did not use any proprietary data or private user information at any stage of the experiments. These datasets are publicly available and widely used in the research community. Consistent with prior work, we did not further anonymize the data, as it was collected and released under publicly accessible licenses and used solely for research purposes. We did not introduce new personal data, nor did we attempt to re-identify individuals beyond what is already present in the original datasets.

\bibliography{main}
\bibliographystyle{unsrtnat}

\clearpage
\appendix
\setlength{\abovedisplayskip}{6pt plus 2pt minus 2pt}
\setlength{\belowdisplayskip}{6pt plus 2pt minus 2pt}
\setlength{\abovedisplayshortskip}{3pt plus 2pt minus 1pt}
\setlength{\belowdisplayshortskip}{3pt plus 2pt minus 1pt}

\etocdepthtag.toc{mtappendix}
\etocsettagdepth{mtchapter}{none}
\etocsettagdepth{mtappendix}{subsection}
{
\parskip=0em
\tableofcontents
}

\section{Usage of AI}
We employ large language models (LLMs) to improve the grammar and clarity of our writing and assist coding. All research ideas, directions, and decisions are independently conceived and carried out by the authors.

\section{Preservation of the $L_2$-Distance and Cosine Similarity of Random Projection}
\label{apx:rp-cos-proof}
The theory has already been proved by \citet{johnson1984extensions}. For completeness, we provide analysis below.
Let the query embedding be $\bm{x}$, the retrieved embedding by vanilla RAG be $\bm{e}$, and the random projection matrix be $R$. The original $\ell_2$ distance is $\lVert \bm{x} - \bm{e} \rVert_2$. 
According to the Johnson--Lindenstrauss transform, when the projection variance is $1/k$, we have the boundary condition
\[\displaystyle
(1 - \lambda_{JL}) \lVert \bm{x} - \bm{e} \rVert_2\leq \lVert R\bm{x} - R\bm{e} \rVert_2\leq (1 + \lambda_{JL}) \lVert \bm{x} - \bm{e} \rVert_2
\]
where $\lambda_{JL} = \Omega\bigl(\sqrt{\tfrac{\log d}{k}}\bigr)$~\citep{johnson1984extensions}.
This ensures that the $\ell_2$ distance is well preserved after projection.

Further, the original retrieval score, typically cosine similarity, $\rho = \frac{\bm{x}^\top \bm{e}}{\|\bm{x}\|_2 \, \|\bm{e}\|_2}$ can also be preserved after projection.  After random projection, the score becomes 
\[
\rho' = \frac{\bm{x}^\top R^\top R \bm{e}}{\|R \bm{x}\|_2 \, \|R \bm{e}\|_2}.
\]
For fixed $x$ and $e$, the expectation satisfies $\mathbb{E}[\rho'] = \rho \left( 1 - \frac{1 - \rho}{2k} \right) + \mathcal{O}\left(\frac{1}{k^2}\right)$ (Proof in \cref{proof:cosine}). Thus, when $k$ is sufficiently large, the expected cosine similarity after projection closely approximates the original score, preserving retrieval quality despite token perturbations. 

\subsection{Proof: Preservation of the Cosine Similarity}
\label{proof:cosine}
Our main argument is that with multiple rounds of token generation, the overall semantic meaning is preserved even if individual tokens are perturbed or replaced. In other words, we have
\begin{equation}
    \label{eqa:rho}
    \mathbb{E}[\rho']\approx\rho \left(1-\frac{1-\rho}{2k}\right).
\end{equation}
\begin{proof}
We have query embedding $\bm{x}$ and datastore embedding $\bm{e}$. First, let $\bm{u}=\frac{\bm{x}}{\lVert \bm{x}\rVert_2}$ and $\bm{v}=\frac{\bm{e}}{\lVert \bm{e} \rVert_2}$ and we can have $\rho=\bm{u}^\top \bm{v}$. With that, we can have 
\[
    \rho'=\frac{(\bm{u}^\top R^\top)(R\bm{v}))}{\lVert \bm{u}^\top R^\top\rVert_2 \lVert R\bm{v} \rVert_2}=\frac{(R\bm{u})^\top(R\bm{v}))}{\lVert R\bm{u}^\top\rVert_2 \lVert R\bm{v} \rVert_2}
\]
We let $\bm{a}=R\bm{u}$ and $\bm{b}=R\bm{v}$. Then we can have 
\[
    \rho'=\frac{\bm{a}^\top \bm{b}}{\lVert \bm{a}\rVert_2 \lVert \bm{b}\rVert_2}.
\]
Because $R$ is an IID Gaussian matrix and each element is independently sampled. We can find rotation matrices $U$ and $V$ where $URV \overset{d}{=} R$. Hence, we can have $R(V\bm{x})\overset{d}{=} R\bm{x}$. So, without loss of generality, we can rotate the basis so that $\bm{u}=\bm{e_1}=(1,0,0,\ldots,0)^\top$, $\bm{v}=\rho \bm{e_1}+\sqrt{1-\rho^2}\bm{e_2}$. We know that each row of $R$ is a Gaussian vector $\bm{g_i}\sim \mathcal{N}(0,I_d/k)$. So, $a_i=\bm{r_i}^T\bm{u}=r_{i1}$ and $b_i=\bm{r_i}^\top \bm{v}=\rho r_{i1}+\sqrt{1-\rho^2}r_{i2}$ where $r_{i1}$, $r_{i2}$ are IID $\mathcal{N}(0,1/k)$. Thus, for each $i=1,\ldots,k$,
\[
\begin{bmatrix}
a_i \\
b_i
\end{bmatrix}
\sim \mathcal{N}\!\left(
0,\ \frac{1}{k}
\begin{bmatrix}
1 & \rho \\
\rho & 1
\end{bmatrix}
\right).
\]
With this bivariate normal distribution, we can define
\[
    X_i=\sqrt{k}a_i \sim \mathcal{N}(0,1),
\]
\[
    Y_i=\sqrt{k}b_i \sim \mathcal{N}(0,1),
\]

And we can have that $\rm{Corr}(X_i,Y_i)=\rho$. Then, we can calculate 
\[
        \rho'=\frac{\sum_i X_i Y_i}{\sqrt{\sum_i X_i^2}\sqrt{\sum_i Y_i^2}}
\]
This is exactly the sample Pearson correlation of $k$ IID bivariate normal samples with population correlation $\rho$.

\begin{lemma}
\label{lemma:bivariate}
    For random variables $X_i$ and $Y_i$ where $X_i\sim \mathcal{N}(0,1)$, $Y_i\sim \mathcal{N}(0,1)$, and $\rm{Corr}(X_i,Y_i)=\rho$, we can calculate the expectation of sample Pearson correlation coefficient 
    \[
        \rho'=\frac{\sum_i X_i Y_i}{\sqrt{\sum_i X_i^2}\sqrt{\sum_i Y_i^2}}
    \] as
    \[
        \mathbb{E}[\rho']=\rho \left(1-\frac{1-\rho}{2k}\right)+\mathcal{O}\left(\frac{1}{k^2}\right)
    \]
\end{lemma}
Although the proof of this lemma is not our original contribution, as there is plenty of literature showing this conclusion~\citep{zimmerman2003bias,hotelling1953new}, we would like to provide a proof here to demonstrate some of our insights. 
\begin{proof}
    Let $S_{xx}=\frac{1}{k}\sum X_i^2$, $S_{yy}=\frac{1}{k}\sum Y_i^2$, $S_{xy}=\frac{1}{k}\sum X_i Y_i$. Then
    \[
        \rho'=\frac{S_{xy}}{\sqrt{S_{xx}}\sqrt{S_{yy}}}
   \]
    If we use Taylor expansion, we will have
    \[
        \frac{1}{\sqrt{S_{xx}}}=1-\frac{1}{2}S_{xx}+\mathcal{O}\left(\frac{1}{k^2}\right)
    \]
    As a result, we will have 
   \begin{align*}
       \mathbb{E}[\rho']&=\mathbb{E}[S_{xy}]-\frac{1}{2}\mathbb{E}[S_{xy}S_{xx}]-\frac{1}{2}\mathbb{E}[S_{xy}S_{yy}]\\&+\frac{1}{4}\mathbb{E}[S_{xx}S_{yy}]+\mathcal{O}\left(\frac{1}{k^2}\right)\\
       &=\rho-\frac{1}{2k}\rm{Cov}(XY,X^2)-\frac{1}{2k}\rm{Cov}(XY,Y^2)\\
       &+\frac{1}{4k}\rm{Cov}(X^2,Y^2)+\mathcal{O}\left(\frac{1}{k^2}\right)
    \end{align*}
    We have $\mathbb{E}[X^2]=\mathbb{E}[Y^2]=1$ and $\mathbb{E}[XY]=\rho$. We also know $\rm{Var}(X)=\rm{Var}(Y)=1$ and $\rm{Cov}(X,Y)=\rho$. Using Isserlis theorem, we have $\mathbb{E}[X^4]=3$, $\mathbb{E}[Y^4]=3$, $\mathbb{E}[X^2Y^2]=1+2\rho^2$, $\mathbb{E}[X^3Y]=\mathbb{E}[XY^3]=3\rho$.
    Hence, we have
    \[
        \rm{Cov}(X^2,Y^2)=\mathbb{E}[X^2Y^2]-\mathbb{E}[X^2]\mathbb{E}[Y^2]=2\rho^2
    \]
    \[
        \rm{Cov}(X^2,XY)=\mathbb{E}[X^3Y]-\mathbb{E}[X^2]\mathbb{E}[XY]=2\rho
    \]
    By putting them back, we have our conclusion.
\end{proof}

According to \cref{lemma:bivariate}, we prove \cref{eqa:rho}. The expected distortion goes down as $\mathcal{O}(1/k)$. Hence, we can see that our designed random projection preserves the cosine similarity.
\end{proof}

\subsection{Synthetic Data}
\label{sec:sanity-check}
\begin{wrapfigure}{r}{0.5\linewidth}
    \centering
    \includegraphics[width=\linewidth]{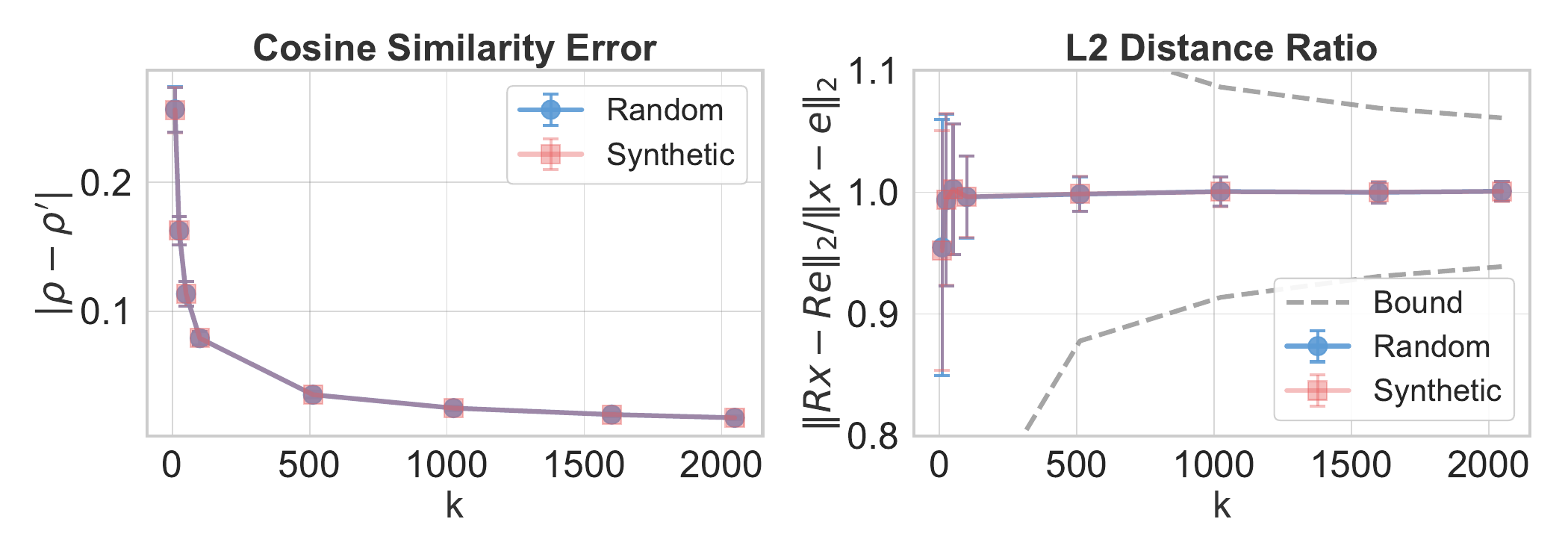}
    \caption{Cosine similarity error and L2 distance ratio under different sanity-check settings.}
    \label{fig:sanity-check}
\end{wrapfigure}
Before starting our experiments, we first conduct a sanity check to verify whether embedding distances remain consistent after random projection. We consider two settings. In the synthetic setting, we use \textsc{Llama3-8B-Instruct} to generate 100 datastore entries with the prompt \emph{“Generate a datastore entry containing domain knowledge.”} For each of these, we repeat 100 times to generate a query using the prompt \emph{“Generate a query about domain knowledge.”} In the random setting, both datastore and query embeddings are random vectors drawn from $\mathcal{N}(0,1)$. 

We measure the L2 distance and cosine similarity between each query embedding and the datastore embeddings. We use two metrics: the L2 distance ratio, $\frac{\lVert R\bm{x} - R\bm{e} \rVert_2}{\lVert \bm{x} - \bm{e} \rVert_2}$, and the cosine similarity difference, $|\rho - \rho'|$. We first average each metric over 100 datastores and then average across 100 runs. \cref{fig:sanity-check} presents the sanity-check results, along with the theoretical upper and lower bounds derived from the Johnson--Lindenstrauss (JL) transform, given by $1 \pm \Omega\left(\sqrt{\frac{\log d}{k}}\right)$. As $k$ increases, the cosine similarity error decreases and approaches zero, confirming our later observation that larger $k$ improves perplexity, which is consistent with our theoretical analysis. For the L2 distance, the results follow the JL transform prediction, with the ratios remaining within the theoretical bounds. A larger $k$ also stabilizes the ratio. Overall, these results verify that, in expectation, distances are well preserved after random projection.

\section{Proof: Probability of Top-1 Selection Changing after Random Projection}
\label{apx:prob-after-rp}
Without loss of generality, we consider the top-1 case, where the datastore embedding $\bm{e}_1$ attains the largest cosine similarity $\rho_1$ in the vanilla RAG setting. We show that if there exists another datastore embedding $\bm{e}_j$ ($j\neq 1$) such that
\begin{align*}
\left\lVert 
\frac{\bm{e}_1}{\lVert \bm{e}_1\rVert_2}
-
\frac{\bm{e}_j}{\lVert \bm{e}_j\rVert_2}
\right\rVert_2
\leq \Delta
\end{align*}
then the probability that the projected cosine similarity $\rho_j'$ exceeds $\rho_1'$ satisfies
\begin{align*}
\Pr(\rho_j' > \rho_1')
\;\gtrsim\;
1 - \Phi\!\left(
\frac{\Delta}
{(1-\rho_1^2)\sqrt{\frac{2(1-\eta_{\max})}{k-3}}}
\right)
\end{align*}
where $\eta_{\max}$ is an upper bound on the correlation between the Fisher-transformed projected similarities, and $\Phi$ denotes the cumulative distribution function of the standard normal distribution.

\begin{proof}
We define the normalized query and datastore embeddings as
\[
\bm{u}=\frac{\bm{x}}{\lVert \bm{x}\rVert_2},
\qquad
\bm{v}_i=\frac{\bm{e}_i}{\lVert \bm{e}_i\rVert_2}.
\]
Under this notation, the assumption becomes $\lVert \bm{v}_1-\bm{v}_j\rVert_2\leq\Delta$. By the Cauchy--Schwarz inequality,
\[
|\rho_1-\rho_j|
=
|\bm{u}^\top(\bm{v}_1-\bm{v}_j)|
\leq
\lVert \bm{v}_1-\bm{v}_j\rVert_2
\leq
\Delta.
\]

We next introduce the Fisher transform
\[
z_i=\tanh^{-1}(\rho_i),
\qquad
z_i'=\tanh^{-1}(\rho_i').
\]
Since cosine similarity takes values in $(-1,1)$ and $\tanh^{-1}(\cdot)$ is strictly increasing on this interval, the ordering of similarities is preserved under this transformation. In particular,
\[
\rho_j' > \rho_1'
\quad\Longleftrightarrow\quad
z_j' > z_1'.
\]

As shown in \cref{apx:rp-cos-proof}, the projected cosine similarity $\rho_i'$ is the sample Pearson correlation computed from $k$ IID bivariate normal pairs with population correlation $\rho_i$. By the classical Fisher $z$-transform result~\cite{fisher1915frequency}, we have
\begin{align*}
z_i' \;\approx\; \mathcal{N}\!\left(\tanh^{-1}(\rho_i),\,\frac{1}{k-3}\right).
\end{align*}
Because $z_1'$ and $z_j'$ share the same projected query vector, they are generally correlated. Let $\eta=\mathrm{Corr}(z_1',z_j')$, and assume $\eta\leq\eta_{\max}$ for some fixed upper bound $\eta_{\max}<1$.

Under a joint Gaussian approximation, the difference $z_j'-z_1'$ is approximately normally distributed with
\[
\mathbb{E}[z_j'-z_1'] = z_j-z_1,
\qquad
\mathrm{Var}(z_j'-z_1') = \frac{2(1-\eta)}{k-3}.
\]
Therefore,
\begin{align*}
\Pr(\rho_j' > \rho_1')
=
\Pr(z_j'-z_1' > 0)
=
\Phi\!\left(
\frac{z_j-z_1}
{\sqrt{\frac{2(1-\eta)}{k-3}}}
\right).
\end{align*}

Since $\rho_1\geq\rho_j$, we have $z_1\geq z_j$. By the mean value theorem, there exists $c$ between $\rho_1$ and $\rho_j$ such that
\begin{align*}
z_j-z_1
=
(\rho_j-\rho_1)\frac{1}{1-c^2}.
\end{align*}
Using $|\rho_j-\rho_1|\leq\Delta$ and $1-c^2\leq 1-\rho_1^2$, we obtain
\begin{align*}
z_j-z_1 \geq -\frac{\Delta}{1-\rho_1^2}.
\end{align*}
Substituting this bound yields
\begin{align*}
\Pr(\rho_j' > \rho_1')
&\gtrsim
\Phi\!\left(
-\frac{\Delta}
{(1-\rho_1^2)\sqrt{\frac{2(1-\eta)}{k-3}}}
\right) \\
&=
1-\Phi\!\left(
\frac{\Delta}
{(1-\rho_1^2)\sqrt{\frac{2(1-\eta)}{k-3}}}
\right) \\
&\gtrsim
1-\Phi\!\left(
\frac{\Delta}
{(1-\rho_1^2)\sqrt{\frac{2(1-\eta_{\max})}{k-3}}}
\right),
\end{align*}
which completes the proof.
\end{proof}

This result shows that after random projection, embeddings that are sufficiently close in the original space can overtake the original top-1 embedding with non-negligible probability. This behavior is desirable for privacy preservation: sensitive tokens (e.g., names, addresses, or identification numbers) often have embeddings close to semantically related but less sensitive alternatives, and random projection allows these alternatives to replace the original sensitive tokens with controlled probability. Combined with the bounds in \cref{apx:rp-cos-proof}, this demonstrates that random projection limits distortion to a reasonable range, preserving general semantic content while reducing the likelihood that specific sensitive tokens are deterministically selected.

\section{Complete Algorithms}
\label{apex:alg}

\begin{algorithm}[!htbp]
\caption{Private RAG with Random Projection (KNN-LM)}
\label{alg:rpknnlm}
\footnotesize
\begin{algorithmic}[1]

\Require Datastore $D$ with $n$ documents;
tokenizer and embedding function
$f(\cdot): \text{Text} \rightarrow \mathbb{R}^d$;
$\lambda=0.25$.
\Statex \textbf{Input:} User query $q$
\Statex \textbf{Output:} Generated answer $a$
\Statex \textbf{Parameters:} Projection dimension $k$;
maximum token length $T_{\max}$

\Statex \textbf{Offline preprocessing}
\State Sample random projection matrix
$R \in \mathbb{R}^{k \times d}$,
where $R_{ij} \sim \mathcal{N}(0,1/k)$
\State Compute datastore embeddings
$E_D \in \mathbb{R}^{d \times n}$,
where $\bm{e}_i = f(d_i)$
\State Project embeddings:
$\tilde{E}_D = R E_D$

\Statex \textbf{Online inference (per query)}
\State Initialize $t \gets 0$,
$a \gets \emptyset$, $y \gets \emptyset$
\While{$t < T_{\max}$ and $y$ is not the end token}
    \State Compute query embedding:
    $\bm{x} = f(q)$
    \State Project query embedding:
    $\tilde{\bm{x}} = R\bm{x}$
    \State $P(y \mid x)
    = \lambda P_{LM}(x)
    + (1-\lambda)P_{KNN}(\tilde{\bm{x}},\tilde{E}_D)$
    \State Predict next token:
    $y \gets \arg\max_y P(y \mid x)$
    \State Append token:
    $a \gets a \Vert y$,
    $q \gets q \Vert y$
    \State $t \gets t+1$
\EndWhile

\State \Return{$a$}

\end{algorithmic}
\end{algorithm}

For KNN-LM (\Cref{alg:rpknnlm}), at each generation step we apply random projection to both the query embedding and the datastore embeddings. The language model produces a next-token distribution $P_{\mathrm{LM}}$ based on the current context $x$. In parallel, the KNN distribution $P_{\mathrm{KNN}}$ is computed by measuring the cosine similarity between the projected query embedding $\tilde{x}$ and each projected datastore embedding $\tilde{e}_i$, followed by a softmax over the top-$K$ nearest neighbors to obtain a probability distribution over candidate tokens. The final next-token distribution is obtained by linearly combining the two distributions with weight $\lambda=0.25$. The selected token $y$ is then appended to the context and the generated answer, and the process repeats.

\begin{algorithm}[!htbp]
\caption{Implementation of DP-RP-G (KNN-LM)}
\label{alg:dprpg}
\footnotesize
\begin{algorithmic}[1]

\Require Datastore $D=\{d_i\}_{i=1}^n$;
embedding function $f:\text{Text}\to\mathbb{R}^d$
\Statex \textbf{Input:} User query $q$
\Statex \textbf{Output:} Generated answer $a$
\Statex \textbf{Parameters:} Projection dimension $k$;
max tokens $T_{\max}$; clipping bound $c$;
privacy $(\epsilon,\delta)$

\Statex \textbf{Offline preprocessing}
\State Sample random projection
$R\in\mathbb{R}^{k\times d}$,
where $R_{ij}\sim\mathcal{N}(0,1/k)$
\State Compute datastore embeddings
$E_D\in\mathbb{R}^{d\times n}$,
with columns $\bm{e}_i=f(d_i)$
\State Project:
$\tilde{E}_D=RE_D\in\mathbb{R}^{k\times n}$

\For{each column $\tilde{\bm{e}}_i$ of $\tilde{E}_D$}
    \State $\hat{\bm{e}}_i
    \gets \tilde{\bm{e}}_i
    \cdot \min\left(1,
    \frac{c}{\|\tilde{\bm{e}}_i\|_2}\right)$
\EndFor

\Statex \textbf{Online inference (per query)}
\State Initialize $t\gets 0$,
$a\gets\emptyset$, $y\gets\emptyset$,
$\sigma=\frac{c\sqrt{2\ln(1.25/\delta)}}{\epsilon}$
\State Sample IID
$Z\in\mathbb{R}^{k\times n}$,
where $Z_{ij}\sim\mathcal{N}(0,\sigma^2)$
\State $\tilde{E}_D\gets\tilde{E}_D+Z$

\While{$t<T_{\max}$ and $y$ is not the end token}
    \State $\bm{x}\gets f(q)\in\mathbb{R}^d$
    \State $\tilde{\bm{x}}\gets R\bm{x}$
    \State $y\gets\arg\max_y P(y\mid x)$
    \State $a\gets a\Vert y$,
    $q\gets q\Vert y$
    \State $t\gets t+1$
\EndWhile

\State \Return{$a$}

\end{algorithmic}
\end{algorithm}

We implement DP-RP-G in \cref{alg:dprpg}, following \citet{blocki2012johnson}, 
which adds Gaussian noise after random projection to achieve differential privacy. 
Unlike \citet{blocki2012johnson}, who assume binary or $[0,1]$-bounded vectors, 
which does not hold for LLM embeddings. 
Therefore, after projection, we apply standard DP clipping to any vector $\bm{v}$:
\[
\operatorname{clip}(\bm{v}, c) = \bm{v} \cdot \min\!\left(1, \frac{c}{\lVert \bm{v} \rVert_2}\right),
\]
so that all vectors have $\ell_2$-norm at most $c$. 
After clipping, we add Gaussian noise calibrated to the privacy budget where the variance $\sigma=c\frac{\sqrt{2\ln(1.25/\delta)}}{\epsilon}$. 
As our goal is to protect the datastore, we apply the DP step only to $E_D$ 
(and not to the queries). Hence, lines 12--19 in \Cref{alg:dprpg} are the same as lines 7--15 in \Cref{alg:rpknnlm}. For the DP-RP-G method implemented in standard RAG, we replace lines 12--19 with lines 7--13 in \Cref{alg:rpllmrag}.

\begin{table*}[!htbp]
\centering
\caption{Comparison of KNN-LM and Standard RAG on Enron Email and ECHR with different $m$ ($\epsilon=5$).}
{
\tablebodyfont
\setlength{\tabcolsep}{2pt}
\begin{tabular}{lcccccc|cccccc}
\toprule
 & \multicolumn{6}{c}{KNN-LM} & \multicolumn{6}{c}{Standard RAG} \\
\cmidrule(lr){2-7} \cmidrule(lr){8-13}
 m& 10 & 20 & 30 & 40 & 50 & 60 
 & 10 & 20 & 30 & 40 & 50 & 60 \\
\midrule

\multicolumn{13}{l}{\textbf{Enron Email}} \\
Privacy (\# leakage) 
& 4 & 2 & 0 & 0 & 0 & 0 
& 8 & 6 & 1 & 0 & 0 & 0 \\
Perplexity 
& 3.14 & 3.15 & 3.14 & 3.14 & 3.16 & 3.16 
& 3.70 & 3.68 & 3.70 & 3.71 & 3.71 & 3.98 \\
\midrule

\multicolumn{13}{l}{\textbf{ECHR}} \\
Privacy (\% leakage) 
& 2.12 & 2.24 & 2.06 & 2.12 & 2.18 & 1.94 
& 2.31 & 2.02 & 2.16 & 2.08 & 2.12 & 2.14 \\
Perplexity 
& 512.64 & 512.64 & 512.64 & 112.35 & 109.64 & 110.35 
& 4307.36 & 4888.23 & 5155.85 & 4307.36 & 7549.47 & 7551.63 \\
\bottomrule
\end{tabular}
}
\label{tab:tunem}
\end{table*}

\section{Implementation Details}
\label{apx:implementation}
First, we use the Enron Email dataset~\citep{klimt2004enron} as the datastore and GPT-2~\citep{radford2019language} as the LLM because GPT-2’s pre-training data do not overlap with Enron Email, ensuring the RAG data remain unseen. Thus, private information is limited to the RAG documents. The dataset contains annotated sensitive information (URLs, email addresses, phone numbers), and we use the entire dataset as the datastore, totaling 4576 sensitive items. We use the query prompts from~\citep{huang2023privacy} (e.g., \texttt{You may also contact us via e-mail at}) and evaluate whether the annotated personally identifiable information can be retrieved.

The European Court of Human Rights (ECHR) dataset~\citep{chalkidis-et-al-2021-ecthr} contains legal case documents from the ECHR. Queries come from test-case descriptions, and utility is measured by matching generated judgments to ground-truth labels. We use \texttt{Llama3-1B-Instruct}. For adversarial queries, we use \texttt{Llama3-8B-Instruct} to annotate sensitive information and query with prompts concatenating \texttt{"This is the first sentence of a paragraph from ECHR legal documents. Complete the paragraph with the sentence:"} and the first sentence of each case. We check if generations contain the annotated sensitive identity, using 10K documents.
  
Besides targeted leakage where we check if exact annotated information is leaked, we use the C4 dataset, a cleaned and large-scale version of Common Crawl, for untargeted leakage evaluation. We select 5,000 samples and use the first sentence of each as the query. We then compute ROUGE-L scores between each generation and every datastore sample. If any sample in the datastore has a score greater than 0.5 with the generation, where we follow the same settings by \citet{huang2023privacy}, marking it as a leakage. We report the number of leakages out of the 5k samples. We use the \textsc{Llama3-8B-Instruct} model as the backbone.

\textsc{HotpotQA}~\citep{yang2018hotpotqa} is a multi-hop QA dataset with strong supervision for supporting facts, enabling more explainable systems. We use it to evaluate whether we retrieve correct supporting documents and measure cosine similarity between retrieved and reference embeddings to assess retrieval quality. We use the training set as the datastore and the test set for evaluation. The datastore contains over 90K documents. We evaluate with distractor setting which means the whole datastore is for reference. We use \textsc{Qwen2-7B-Instruct} as the backbone model. Since document retrieval relies on KNN for both KNN-LM and standard RAG, we focus on the KNN-LM structure to analyze the retrieval process. For evaluation, we report exact match (EM) and F1 scores between the generated output and the ground truth. As the datastore is very large, we randomly select 100 documents as targets for targeted attacks. The initial query in multi-round attack is each question.

\begin{lrbox}{\distributiontable}
\begin{minipage}{\columnwidth}
\begin{table}[H]
\centering
\captionof{table}{Perplexity and data leakage comparison using random projection matrices drawn from different distributions on the Enron Email dataset with GPT2.}
\label{tab:rademacher}
\tablebodyfont
\setlength{\tabcolsep}{2pt}

\begin{tabular}{cclcccccccc}
\toprule
& & \multicolumn{4}{c}{KNN-LM} & \multicolumn{4}{c}{Standard RAG} \\
\midrule
& & Perplexity & Email & URL & Phone & Perplexity & Email & URL & Phone \\
\midrule
\multicolumn{2}{c}{RAG}
& 2.87 & 0 & 13 & 0
& 1.58 & 36 & 56 & 125 \\
\midrule
\multirow{2}{*}{RP (Ours)}
& Gaussian
& \textbf{2.89} & 0 & 0 & 0
& 2.05 & 0 & 3 & 0 \\
& Rademacher
& 2.92 & 0 & 2 & 0
& \textbf{1.76} & 0 & 3 & 1 \\
\bottomrule
\end{tabular}%
\end{table}
\end{minipage}
\end{lrbox}

\begin{table*}[!b]
\centering
\captionof{table}{Example responses generated by different methods given the prompt.}
\label{tab:example}
{
\tablebodyfont
\setlength{\tabcolsep}{2pt}
\begin{tabular}{p{\dimexpr0.18\linewidth-4pt\relax} p{\dimexpr0.82\linewidth-4pt\relax}}
\toprule
\textbf{Prompt} &
\textit{Complete the sentences: you have any suggestions, please share. With the rise of social media...} \\
\midrule

\textbf{Parametric} &
With that in mind  here are a few ideas to get you started:  1. **Create a vision board**: Cut out images and words that represent your goals and desires  and create a visual representation of what you want to achieve. 2. **Set SMART goals**: Make sure your goals are Specific Measurable Relevant and Time-bound. This will help you stay focused and motivated. 3. **Break down big goals into smaller.\\
\midrule

\textbf{RAG} &
it's become increasingly important for businesses to have a strong online presence. One way to do this is by creating engaging content that resonates with your target audience. Here are some tips to help you create content that drives engagement:  1. Know your audience: Before you start creating content make sure you have a good understanding of your target! audience. Who are they? What are their interests? What problems do they need?\\
\midrule

\textbf{RP (Ours)} &
it's become increasingly important for businesses to have a strong online presence. One way to do this is by creating engaging content that resonates with your target audience. Here are some tips to help you create content that drives engagement:  1. Know your audience: Before you start creating content  make sure you have a good understanding of who your target! audience is. What are their interests need?\\
\midrule

\textbf{DP-RP-G ($\epsilon=20$)} &
PrintWriter.PrintWriter userService userServiceExchangeExchangeExchange\_generate\_generate\\
\midrule

\textbf{DPSparseVoteRAG ($\epsilon=20$)} &
make sure you have a good understanding of your target! "\#"\!"\#\!"!\$""\!\&""\!\$"\#"\#\!"\*\\
\bottomrule
\end{tabular}
}
\end{table*}

\section{Settings of Hyper-Parameters}
\label{apx:hyper}
\subsection{General Settings}
For all LLM generations, we set the repetition penalty to 0.75 and the no-repeat-ngram-size to 0. The feature dimensions of GPT-2, Llama3, and Qwen2 are 128, 2048, and 3584 respectively. All other generation parameters follow the default \textsc{HuggingFace} settings. By default, we set the datastore size equal to the total number of tokens in each dataset. We follow previous differential privacy work where we set $\delta$ as the reverse scale of dataset size. For Enron Email and ECHR, it is $10^{-4}$. For HotPotQA and C4, we use $10^{-5}$. We use each model's tokenizer and its last hidden-state vectors as embeddings for both retrieval settings. The max output token length is 2048.

\subsection{DPSparseRAG}
\label{apx:tunem}
{Apart from $\epsilon$ and clipping bound we tune on in the experiments, another hyper-parameter for DPSparseRAG is $m$. We therefore conduct a sweep over 
$m\in\{10,20,30,40,50,60\}$ with $\epsilon=5$ on the Enron Email and ECHR datasets. Overall, we observe in \Cref{tab:tunem} no significant difference in utility across various values of $m$.
}

\par\noindent\usebox{\distributiontable}\par

\section{Random Projection with Other Distributions}
\label{apx:other-distributions}
Beyond the normal (Gaussian) distribution, we also experimented with constructing the projection matrix using a Rademacher distribution. In this distribution, each entry independently takes the value $+1$ or $-1$ with equal probability (50\%). As shown in \cref{tab:rademacher}, the Rademacher-based projection also preserves privacy effectively while achieving performance comparable to vanilla RAG. These results indicate that our proposed random projection method remains effective as long as the projection matrix preserves similarity structure, regardless of the specific distribution.

\section{More Discussion}
\label{apx:morediscussion}
\subsection{False Positive Classification of Non-Sensitive Information}
Although random projection provides strong privacy protection overall, our experiments on Enron Email, ECHR, and C4 reveal that a small number of items may still be leaked. All failure cases fall into the case of leakage examples without sensitive information. Analysis of these failure cases reveals a clear pattern: random projection provides complete protection (100\% leakage prevention) for specific content such as numbers, IDs, and email addresses. 

While, in rare cases, common terminology which are non-sensitive information may mis-classified personal identity information in annotation (e.g. common streat name Queen St.) This occurs because such terms may appear in multiple contexts, making them less distinctive in the embedding space. These findings suggest that random projection is most effective for protecting unique identifiers while providing strong but not absolute protection for common terms that may appear across multiple documents.

\subsection{An Example of Generation}
\Cref{tab:example} provides an example of responses generated by different methods for the same prompt. The parametric model produces unrelated continuation that is not grounded in the datastore. In contrast, both RAG and our method generate clear, relevant answers, whereas the DP-based methods often produce incoherent or unrelated content.

\end{document}